\documentclass{article}
\usepackage{graphicx} 
\usepackage{geometry}
\usepackage{authblk}

\usepackage[suppress]{color-edits}
\addauthor[Christina]{cy}{violet}
\addauthor[Xumei]{xx}{blue}
\addauthor[Yudong]{yc}{red}

\usepackage[utf8]{inputenc} 
\usepackage[T1]{fontenc}    
\usepackage[colorlinks]{hyperref}       
\usepackage{url}            
\usepackage{booktabs}       
\usepackage{amsfonts}       
\usepackage{nicefrac}       
\usepackage{microtype}      

\usepackage{mathtools}
\usepackage{amsmath}
\usepackage{amsthm}
\usepackage{amssymb}
\usepackage[ruled,linesnumbered]{algorithm2e}
\usepackage{dsfont}
\usepackage{bm}

\makeatletter
\theoremstyle{plain}
\newtheorem{thm}{\protect\theoremname}
\theoremstyle{plain}
\newtheorem{lem}{\protect\lemmaname}
\theoremstyle{plain}
\newtheorem{corollary}{\protect\corollaryname}
\theoremstyle{remark}

\theoremstyle{plain}

\theoremstyle{remark}
\newtheorem*{rem*}{\protect\remarkname}
\theoremstyle{plain}
\newtheorem*{assum*}{\protect\assumptionname}
\theoremstyle{plain}
\newtheorem{defn}{\protect\definitionname}

\allowdisplaybreaks

\makeatother

\usepackage{babel}
\providecommand{\lemmaname}{Lemma}
\providecommand{\propositionname}{Proposition}
\providecommand{\remarkname}{Remark}
\providecommand{\theoremname}{Theorem}
\providecommand{\assumptionname}{Assumption}
\providecommand{\definitionname}{Definition}
\providecommand{\corollaryname}{Corollary}

\title{Matrix Estimation for Offline Reinforcement Learning with Low-Rank Structure}

\author[1]{Xumei Xi}
\author[1]{Christina Lee Yu}
\author[2]{Yudong Chen}
\affil[1]{School of Operations Research and Information Engineering, 
Cornell University}
\affil[2]{Department of Computer Sciences, University of Wisconsin-Madison}

\date{\texttt{ \{xx269, cleeyu\}@cornell.edu, yudong.chen@wisc.edu }}

\begin{document}

\global\long\def\norm#1{\left\Vert #1\right\Vert }%
\global\long\def\card#1{\left|#1\right|}%
\global\long\def\R{\mathbb{R}}%
\global\long\def\opnorm#1{\norm{#1}_{\mathrm{op}}}%
\global\long\def\infnorm#1{\norm{#1}_{\infty}}%
\global\long\def\cutnorm#1{\norm{#1}_{\square}}%
\global\long\def\twonorm#1{\norm{#1}_{2}}%
\global\long\def\nunorm#1{\norm{#1}_{\ast}}%
\global\long\def\fnorm#1{\norm{#1}_{F}}%
\global\long\def\onenorm#1{\left\Vert #1\right\Vert _{1}}%
\global\long\def\maxnorm#1{\left\Vert #1\right\Vert _{\text{max}}}%
\global\long\def\ip#1{\left\langle #1\right\rangle }%
\global\long\def\trace{\operatorname{Tr}}
\global\long\def\argmax{\operatorname{argmax}}%
\global\long\def\poly{\operatorname{poly}}%
\global\long\def\pihat{\widehat{\pi}}%
\global\long\def\pistar{\pi^{\ast}}%
\global\long\def\subopt{\operatorname{SubOpt}}%

\global\long\def\statespace{\mathcal{S}}%
\global\long\def\statesize{S}%
\global\long\def\actionspace{\mathcal{A}}%
\global\long\def\actionsize{A}%
\global\long\def\rewardbound{R}%
\global\long\def\horizon{H}%
\global\long\def\discount{\gamma}%
\global\long\def\E{\mathbb{E}}%
\global\long\def\indic{\mathds{1}}%
\global\long\def\dataset{\mathcal{D}}%
\global\long\def\P{\mathbb{P}}%
\global\long\def\supp{\operatorname{supp}}%
\global\long\def\transition{P}%
\global\long\def\target{\pi^{\theta}}%
\global\long\def\behavior{\pi^{\beta}}%
\global\long\def\one{\boldsymbol{1}}%
\global\long\def\mmu{\boldsymbol{\mu}}%
\global\long\def\pphi{\boldsymbol{\phi}}%
\global\long\def\rank{\operatorname{rank}}%
\global\long\def\diag{\operatorname{diag}}%
\global\long\def\ME{\mathtt{ME}}%
\global\long\def\argmin{\operatorname{argmin}}%
\global\long\def\sgn{\operatorname{sgn}}%
\global\long\def\pmin{p_{\min}}%
\global\long\def\pthr{p_{\mathrm{thr}}}%
\global\long\def\cZ{\mathcal{Z}}%
\global\long\def\normal{\mathcal{N}}%
\global\long\def\dis{\operatorname{Dis}}
\global\long\def\dishat{\operatorname{\widehat{Dis}}}
\global\long\def\polylog{\operatorname{polylog}}
\global\long\def\Diff{\operatorname{Diff}}
\global\long\def\B{\bm{B}}

\maketitle

\begin{abstract}
      We consider offline Reinforcement Learning (RL), where the agent does not interact with the environment and must rely on offline data collected using a behavior policy.
      Previous works provide policy evaluation guarantees when the target policy to be evaluated is covered by the behavior policy, that is, state-action pairs visited by the target policy must also be visited by the behavior policy.
      We show that when the MDP has a latent low-rank structure, this coverage condition can be relaxed. Building on the connection to weighted matrix completion with non-uniform observations, we propose an offline policy evaluation algorithm that leverages the low-rank structure to estimate the values of uncovered state-action pairs. Our algorithm does not require a known feature representation, and our finite-sample error bound involves a novel discrepancy measure quantifying the discrepancy between the behavior and target policies in the spectral space. We provide concrete examples where our algorithm achieves accurate estimation while existing coverage conditions are not satisfied. Building on the above evaluation algorithm, we further design an offline policy optimization algorithm and provide non-asymptotic performance guarantees. 
\end{abstract}

\section{Introduction}
\label{sec:intro}
    
    Reinforcement Learning (RL) has achieved significant empirical success in the online setting, where the agent continuously interacts with the environment to collect data and improve its performance.
    However, online exploration is costly and risky in many applications, such as healthcare~\cite{gottesman2019medicine} and autonomous driving~\cite{shalev2016autonomous}, in which case it is preferable to learn from a pre-collected observational dataset from doctors or human drivers using their own policies.
    Due to lack of on-policy interaction with the environment, offline RL faces the fundamental challenge of distribution shift~\cite{levine2020survey}.
    A standard approach for handling distribution shift is importance sampling~\cite{precup2000importance, precup2001importance}. More sophisticated approaches have been proposed to alleviate the high variance of importance sampling~\cite{farajtabar2018robust, wang2017cb}. Recent works~\cite{sutton2016td, nachum2019dualdice, zhang2020gendice} consider estimating the state marginal importance ratio, a more tractable problem.
    
    Existing work on offline RL requires the dataset to have sufficient coverage. A standard measure for coverage is the concentrability coefficient~\cite{uehara2021partial}: $C^\pi = \max_{s,a} \frac{d^\pi(s,a)}{\rho(s,a)}$, which is the ratio between the state-action occupancy measure of a policy $\pi$ of interest and the (empirical) occupancy measure $\rho$ of the behavior policy generating the offline dataset. However, this can be restrictive as the support of $\rho $ must contain that of $d^\pi$ in order for $C^\pi$ to be finite. 
    Earlier work such as the Fitted Q-iteration (FQI) algorithm~\cite{munos2008fqi} requires full coverage, i.e.~$C^\pi < \infty$ for all policies $\pi$. 
    More recent works~\cite{uehara2021partial, rashidinejad2021pessimism, liu2020batch} requires a more relaxed, partial coverage condition $C^{\pi^\ast} < \infty$ with $\pi^\ast$ being optimal policy. Partial coverage is still a fairly strong requirement: the behavior policy must visit every state the optimal policy would visit, and take every action the optimal policy would take.

    In this paper, we seek to relax the coverage condition for offline policy evaluation in settings where the Markov decision process (MDP) has a latent low-rank structure.
    Similarly to \cite{shah2020matrix, sam2023overcoming}, we view the $Q$ function as a matrix and exploit its low-rank structure to infer the entries that were not observed in the offline data. Unlike typical results from the low-rank matrix completion literature, our setting requires completing the matrix under non-uniform sampling, as in \cite{foucart2021weighted, lee2013mc}; moreover, the error is evaluated under a different distribution or weighted norm, leading to the fundamental challenge of distribution shift. By leveraging techniques from weighted and non-uniform matrix completion, we develop a new offline policy evaluation algorithm, which alternates between Q iteration and matrix estimation. For both the infinite and finite sample settings, we show that the evaluation error can be bounded in terms of a novel discrepancy measure between the behavior and target policies. In contrast to the standard concentrability coefficient, our discrepancy measure may remain finite even when the behavior policy does not cover the support of the target policy. We present a concrete example where the concentrability coefficient is infinite but our method achieves a meaningful error bound. Building on the above evaluation algorithm, we further design an offline policy optimization algorithm with provable performance guarantees.

    \xxedit{There are several challenges that arise when we borrow ideas from low-rank matrix estimation to offline RL. 
    Classical matrix estimation results require two assumptions that are hard to satisfy in MDP. 
    First, independence is assumed between the sampling process and the observation noise. This is clearly not true for MDPs, where observation noise is intertwined with the sampling. For example, if a state-action pair is sampled more frequently, the empirical observations~(e.g., transition frequencies) are bound to be less noisy than others sampled less often. 
    Second, sampling in matrix estimation typically requires each entry to have a non-zero probability of being observed. Sampling in MDPs is very different: only entries on the support of the sampling distribution, which is determined by the behavior policy, can be sampled; those off the support have a zero observation probability.
    We note that various recent works attempt to relax the aforementioned assumptions to make matrix estimation more applicable to real-world sequential decision-making problems. 
    For example, the paper~\cite{agarwal2021causal} allows for some dependence between the noise and sampling pattern, and the smallest sampling probability can be zero. Their algorithm, which involves finding a maximum biclique, works best with datasets with a specific restrictive structure, which is often not present in offline RL. Our goal in this paper is to derive a performance guarantee for a more general class of sampling patterns. }

\section{Problem Setup}
\label{sec:setup}

    \subsection{MDP with Low-Rank Structure}

        Consider an MDP $\mathcal{M}=(\statespace,\actionspace,H,\transition,r,\mu_{1}) $
        with finite state space $\statespace$, finite action space $\actionspace$,
        horizon $H$, transition kernel $\transition=\{ P_{t}\} _{t\in[H]}$, bounded reward function $r=\left\{ r_{t}:\statespace\times\actionspace \to [0,1]\right\} _{t\in[H]}$, and initial
        state distribution $\mu_{1}\in\Delta(\statespace)$. 
        Let $\statesize=|\statespace|$ and $\actionsize=|\actionspace|.$
        For each policy $\pi=\{ \pi_{t}: \statespace \to \Delta(\actionspace) \} _{t\in[H]}$, the Q function $Q_{t}^{\pi}:\statespace\times\actionspace\to\R$ is
        defined as $Q_{t}^{\pi}(s,a)=\E_\pi [ \sum_{i=t}^{H}r_{i}(s_{i},a_{i}) \vert s_{t}=s,a_{t}=a ],$
        and the total expected reward is
        $J^{\pi}=\E_\pi [\sum_{t=1}^{H}r_t(s_{t},a_{t})\vert s_{1}\sim\mu_{1} ].$
        Let $d_{t}^{\pi}:\statespace\times\actionspace\to[0,1]$
        denote the state-action occupancy measure at time $t\in[H]$ under
        policy $\pi$. 
        
        Given a dataset generated by the behavior policy $\behavior$, our goal is to estimate $J^{\target}$ for a target policy~$\target$. Our blanket assumption is that the MDP has the following low-rank structure, which implies that for any policy $\pi$, its $Q$ function (viewed as an $S$-by-$A$ matrix) is at most rank $d$. 
        \begin{assum*}
        \label{low_rank_assumption}
            For all $t$, $r_t \in [0,1] ^ {\statesize \times \actionsize}$ has rank at most $d' = \lfloor d/2 \rfloor$, and $P_t$ admits the decomposition 
            \[
                \transition_t(s'\vert s,a) = \textstyle \sum_{i=1}^{d'} u_{t, i}(s',s) w_{t, i}(a)
                \quad\text{or}\quad
                \transition_t(s'\vert s,a) = \textstyle \sum_{i=1}^{d'} u_{t, i}(s) w_{t, i}(s',a), 
                \quad \forall s',s,a.
            \]
        \end{assum*}
        The above low-rank model is different from the Low-rank MDP model considered in previous works~\cite{agarwal2020flambe, uehara2021representation}. In Low-rank MDPs, the transition kernel $P$ is assumed to have a factorization of the form $P(s' \vert s,a) = \sum_{i=1}^{d'} u_i(s')  w_i(s,a)$, where the factors $u_i(\cdot)$ and $w_i(\cdot,\cdot)$ are unknown. Closely related is the Linear MDP model~\cite{jiang2017contextual, yang2020reinforcement}, where the feature maps $w_i(\cdot,\cdot)$ are known. 
        In these models, the low-rank/linear structures are with respect to the relationship between the originating state-action pair $(s,a)$ and the destination state  $s'$; they do \emph{not} imply that $Q$ function is low-rank when viewed as a matrix. In contrast, our model stipulates that the transition kernel can be factorized either between (i) $a$ and $(s,s')$ or (ii) $s$ and $(s',a)$, both of which imply a low
        dimensional relationship between the current state $s$ and the action $a$ taken at that state, resulting in a low-rank $Q$ function. 
        \xxedit{A key consequence of the $Q$ function being low-rank is that we can bypass modeling the environment and directly estimate the $Q$ function by leveraging its low-rankness, resulting in a model-free method. 
        On the other hand, most works in Low-rank MDPs consider model-based methods. Note that when the transition tensor is fully factorized: $P_t(s' \vert s,a) = \sum_{i=1}^{d'} u_{t,i}(s') v_{t,i}(s) w_{t,i}(a)$, it satisfies both our assumption and the assumption of Low-rank MDPs. }

    \subsection{Offline Dataset}
    \label{subsub:offline_dataset}
    
        \xxedit{The offline dataset is denoted by $\dataset=\{ (s_{t}^{k},a_{t}^{k},r_{t}^{k})\} _{t\in[H], k\in[K]}$, which contains $K$ independent trajectories generated from the behavior policy $\behavior$. We consider two settings: the infinite-sample setting with $K \to \infty$ and the finite-sample setting with $K < \infty$; we describe these two settings in detail below. For simplicity, we assume the immediate rewards are observed without noise, which can be easily relaxed. The uncertainty in the system completely comes from the transition probability. 
        
        In the infinite-sample setting, we have partial but noiseless knowledge of the MDP on the support of the state-action occupancy measure induced by the behavior policy. In other words, for all state-action pairs that can be visited using the behavior policy, namely, $(s,a) \in \supp(d_t^{\behavior})$, we know the exact values of the transition probability $P_t(\cdot \vert s,a)$. 
        It is important to note that even in this idealized setting, off-policy evaluation is still non-trivial. When the behavior policy does not have full coverage, i.e., $\supp(d_t^{\behavior}) \neq \statespace \times \actionspace$, we do not have any information for the state-action pairs off the support and they must be estimated by leveraging the low-rank structure of the $Q$ function. 
        The distribution shift that arises in the infinite-sample setting can be attributed to the difference in support, which is precisely reflected in our proposed distribution discrepancy in Definition \ref{defn:op_dis} and the corresponding error bound in Theorem~\ref{thm:dt_bound}. 

        In the finite-sample setting, we have a noisy and unbiased estimate $\widehat P_t(\cdot \vert s,a)$ of the true transition probability $P_t(\cdot \vert s,a)$ for $(s,a) \in \supp(\widehat{d}_t^{\behavior})$, where $\widehat{d}_t^{\behavior}$ denotes the empirical data distribution of $K$ independent samples from the true distribution $d_t^{\behavior}$. 
        Since different estimates of the probability exhibit different levels of uncertainty, only considering the support is no longer sufficient. In particular, the finite-sample distribution shift depends not only on the difference in support, but also the difference in the specific distributions, which will be reflected in our proposed distribution discrepancy in Definition \ref{eq:defn_finite-sample_discrepancy} and the subsequent error bound in Theorem~\ref{thm:dt_bound_sample}. 
        }


    \subsection{Notation and Operator Discrepancy}
        For a matrix $M \in \R^{n\times m}$, let $\nunorm M$ denote its nuclear norm (sum of singular values), $\opnorm M$ its operator norm (maximum singular value), $\norm{M}_\infty=\max_{i,j}\card{M_{ij}}$ its entrywise $\ell_\infty$ norm, and $\supp(M)=\{(i,j): M_{ij} \neq 0\}$ its support. \xxedit{The max norm~\cite{srebro2005rank} of $M$ is defined as $ \norm{M}_{\max} = \min_{U,V: X = UV^\top} \norm{U}_{2 \to \infty } \norm{V}_{2 \to \infty }$,
        where $\norm{\cdot}_{2 \to \infty }$ denotes the maximum row $\ell_2$ norm of a matrix. 
        \ycedit{Both max norm and nuclear norm can be viewed as convex surrogates of the matrix rank~\cite{srebro2005rank}.}
        For a rank-$d$ matrix $M$, its max norm can be upper bounded as
        \begin{align}
        \label{eq:maxnorm_infty_bound}
            \norm{M}_{\max} \le \sqrt{d} \norm{M}_{\infty}.
        \end{align}
        The nuclear norm and the max norm satisfy: 
        \begin{align}
        \label{eq:maxnorm_nuclearnorm}
            \frac{1}{\sqrt{nm}} \nunorm{M} \le \norm{M}_{\max} \le \nunorm{M}. 
        \end{align}}%
        The indicator matrix $\indic_M \in \{0,1\}^{n\times m}$ is a binary matrix encoding the position of the support of $M$.
        The entrywise product between two matrices $M$ and $M'$ is denoted by $M \circ M'$. 
        
        We propose a novel discrepancy measure defined below, 
        and show that it can replace the role of the concentrability coefficient in our infinite-sample error bound under the low-rank assumption.
        
        \begin{defn}[Operator discrepancy]
        \label{defn:op_dis}
           The operator discrepancy between two probability distributions $p,q \in \Delta(\statespace \times \actionspace)$ is defined as
            \begin{align} 
            \label{eq:defn_dis_pq}
                \begin{split}
                \dis(p, q) \coloneqq \min & \Big\{ \opnorm{g-q} : g \in \Delta (\statesize \times \actionsize), \; \supp(g) \subseteq \supp(p) \Big\}.
                \end{split}
            \end{align}
        \end{defn}
        
        Note that $\dis(p, q) \le \opnorm{p-q}$ is always finite, and $\dis(p, q)=0$ if and only if $\supp(q) \subseteq \supp(p)$.
        To provide intuition for $\dis(p, q)$, let the minimizer in~\eqref{eq:defn_dis_pq} be $g^\ast$. By generalized H\"older's inequality,
        \begin{equation} 
        \label{eq:operational_interpretation}
            \Big| \E_{(s,a)\sim g^\ast} \big[M(s,a)\big] - \E_{(s,a)\sim q} \big[M(s,a)\big]  \Big|
            = \Big| \langle g^\ast, M \rangle - \langle q, M \rangle \Big|
            \le \dis(p,q) \cdot \| M \|_*.
        \end{equation}
        If the nonzero singular values of $M$ are of the same scale, then the RHS of~(\ref{eq:operational_interpretation}) is of order $ \dis(p, q) \cdot \textnormal{rank}(M)$. 
        Therefore, $\dis(p, q)$ measures the distribution shift between $p$ and $q$ in terms of preserving the expectation of low-rank matrices. \xxedit{Compared to traditional measures such as the concentrability coefficient, the operator discrepancy takes into account the low-rank structure of the model and therefore can allow for a less restrictive coverage condition. } Note that $\dis(p, q)$ only depends on the support of $p$: if $\supp(p)=\supp(p')$, then $\dis(p, q) = \dis(p', q)$ for all $q$. \xxedit{As mentioned before, the infinite-sample distribution shift depends only on the support of the behavior policy and the operator discrepancy reflects exactly that. }Moreover, thanks to the minimization in the definition~(\ref{eq:defn_dis_pq}), $\dis(p, q)$ can be significantly smaller than $\opnorm{p-q}$. For instance, if $p$ is the uniform distribution on $\statespace\times\actionspace$, then $g^\ast=q$ and hence $\dis(p, q)=0$ for all $q$.
    
        \xxedit{For our finite-sample error bound, we consider a different notion of discrepancy. \cyedit{As we have a limited number of samples from the behavior policy, it is natural that the appropriate notion of discrepancy no longer depends only on the support of the distribution, but rather how closely the distributions match. In our analysis, the appropriate measure of closeness is given by the operator norm difference, which is also closely tied to Definition~\ref{defn:op_dis} when we restrict $g$ to be equal to $p$.}
        \begin{defn}[Empirical Operator Discrepancy]
        \label{eq:defn_finite-sample_discrepancy}
        The empirical operator discrepancy between two probability distributions $p,q \in \Delta(\statespace \times \actionspace)$ is defined as
            \begin{align} 
            \label{eq:defn_finite-sample_dis}
                \begin{split}
                \dishat(p,q) \coloneqq \opnorm{p-q}.
                \end{split}
            \end{align}
        \end{defn}
        }
    
        \xxedit{When infinite samples are given from the behavior policy, the error bound for our proposed off-policy evaluation algorithm will be a function of $\dis(d_t^{\behavior}, d_t^{\target})$.} The operator discrepancy only depends on the support of the behavior policy and not the exact distribution, which is expected under the infinite-sample setting. As such, the operator discrepancy highlights the inherent error induced by distribution shift.
        \xxedit{In the finite-sample setting, our error bound depends on the empirical operator discrepancy $\dishat({d}_t^{\behavior}, d_t^{\target})$, for which the exact distribution matters. This is expected since we are given observations with varying noise levels determined by the empirical data distribution.} 
        
        \ycedit{We remark in passing that the inequality $\dis(d_t^{\behavior}, d_t^{\target}) \le \dishat({d}_t^{\behavior}, d_t^{\target})$ holds by definition. Also, the above discrepancy metrics share similarity with the parameter $\Lambda$ in~\cite{lee2013mc}, which also measures the difference in two distributions in the operator norm. }


\section{Algorithm}
\label{sec:algo}

    In this section, we present our algorithm for offline policy evaluation. The algorithm takes as input an offline dataset $\dataset=\{ (s_{t}^{k},a_{t}^{k},r_{t}^{k})\} _{t\in[H], k\in[K]}$, which contains $K$ independent trajectories generated from the behavior policy $\behavior$. \cyedit{The algorithm also takes as input the target policy $\target$, the initial state distribution $\mu_1$, weight matrices $(\rho_t)_{t\in[H]}$, and a matrix estimation algorithm $\ME(\cdot)$ which will be specified in \eqref{eq:nuclear_norm_min_infinite} and \eqref{eq:max_norm} for the infinite-sample and finite-sample settings, respectively. The weight matrices $\{\rho_t\}$ are chosen by the user and primarily used as an input to $\ME(\cdot)$.}
    \ycedit{As a typical choice, in the infinite-sample setting we set $\rho_t$ to be the true state-action occupancy measure $d_t^{\behavior}$ of the behavior policy; in the finite-sample setting we set $\rho_t$ to be the empirical measure $\widehat{d}_t^{\behavior}$. }


    Our algorithm \cyedit{iterates backwards in the horizon from steps $t = H$ to $t=1$. For each step $t$, the algorithm has two parts. First, it} applies Q-value iteration to empirically estimate the Q-values for state-action pairs in the support of $d_t^{\behavior}$. In particular, the data is used to construct unbiased empirical estimates of the transition kernel and occupancy measure of the behavior policy, denoted by $\widehat{P}_t$ and $\widehat{d}_t^{\behavior}$, respectively. Let $\widehat{B}_t^{\target}$ denote the target policy's empirical Bellman operator, which is given by 
    \begin{align}
    \label{eq:bellman}
        (\widehat{B}_{t}^{\target} f )(s,a)
        = r_{t}(s,a) + {\textstyle\sum_{s',a'}} \widehat{P}_{t}(s'\vert s,a)\target_{t}(a'\vert s')f(s',a')
    \end{align}
    for all $f: \statespace \times \actionspace \to \R$. Note that we can evaluate $(\widehat{B}_{t}^{\pi} f )(s,a)$ only over $(s,a)\in\supp(\widehat{d}_t^{\behavior})$. \cyedit{With the given weight matrix $\rho_t$, which is chosen such that $\supp(\rho_t) \subseteq \supp(\widehat{d}_t^{\behavior})$, the in-support empirical estimate of the Q-value is computed via
    \[Z_{t}(s,a) \gets ( \widehat{B}_{t}^{\target}\widehat{Q}_{t+1}^{\target} ) (s, a), \quad \text{for }(s,a)\in \supp (\rho_t).\]}
    
    Subsequently, to infer the Q values off support, the algorithm uses the low-rank matrix estimation subroutine, $\ME (\cdot)$, \cyedit{which takes as input the weight matrix $\rho_t$ and the empirical estimates $Z_t$}. While $\ME (\cdot)$ \ycedit{can be any off-the-shelf matrix estimation algorithm, for the purpose of the analysis, we will use the max norm minimization method due to its computational tractability and robustness under nonuniform sampling. Specifically, our $\ME (\cdot)$ subroutines are specified in \eqref{eq:nuclear_norm_min_infinite} and \eqref{eq:max_norm} in the next section, in which different constraints are used for the infinite-sample and finite-sample settings}. 

    The pseudocode for our algorithm is given below. Our algorithm is computationally efficient and easy to implement.

    \RestyleAlgo{ruled}
    \begin{algorithm}
        \SetAlgoLined
        \KwData{dataset $\dataset$, $\target$, initial state distribution $\mu_1$, weight matrices $(\rho_t)_{t\in[H]}$,  and $\ME(\cdot)$.}
        \KwResult{estimator $\widehat J$.}
        $\widehat Q_{H+1}^{\target}(s,a) \gets 0, \quad \forall (s,a)\in\statespace \times \actionspace $. \\
        \For{t = H, H-1, \dots, 1}{
        	Q iteration: $ Z_{t}(s,a) \gets ( \widehat{B}_{t}^{\target}\widehat{Q}_{t+1}^{\target} ) (s, a)$, for all $(s,a)\in \supp (\rho_t)$. \\
        	Matrix estimation: $\widehat{Q}_{t}^{\target} \gets \ME \left(\rho_t, Z_{t}\right)$.
        }
        Output $\widehat J \gets \sum_{s,a} \mu_1(s) \target_1 (a \vert s) \widehat{Q}_1^{\target} (s,a).$ 
        \caption{Matrix Completion in Low-Rank Offline RL \label{alg:MC_low-rank_infinite} }
    \end{algorithm}

\section{Analysis}
\label{sec:analysis}
    
    We present evaluation error bounds under both the \emph{infinite-sample} setting $K \to \infty$ and the \emph{finite-sample} setting $K<\infty$. Define the population Bellman operator $B_t^{\target}$, which is given by equation~\eqref{eq:bellman} with $\widehat{P}_t$ replaced by  $P_t$. 
    Define the matrix $Y_t \in \R^{S\times A}$ via $Y_t(s,a) = ({B}^{\target}_{t} \widehat{Q}_{t+1}^{\target}) (s, a)$, which is the population version of $Z_t$ computed in Algorithm~\ref{alg:MC_low-rank_infinite}.

    \subsection{Infinite-sample setting}

         In the infinite-sample setting, we have $\widehat{d}_t^{\behavior}(s,a) \to d_t^{\behavior}(s,a)$ and $\widehat{P}_t(s,a) \to P_t(s,a)$ for all $(s,a)\in\supp(d_t^{\behavior})$. Consequently, both  $\widehat{B}_t^{\target}$ and $Z_t$ converge to their population versions $B_t^{\target}$ and $Y_t$, respectively. \xxedit{Note that the infinite samples does not imply complete knowledge of the MDP. Instead, we only know a subset of transition probabilities on the support of the state-action occupancy measure induced by the behavior policy. 
         The matrix estimation subroutine is given by the following max norm minimization program with $
         \rho_t = d_t^{\behavior}$ and $L_t \coloneqq H-t+1$:
        \begin{align}
            \begin{split}
            \ME( \rho_t, Y_t ) = \underset{M \in\R^{\statesize\times\actionsize}}{\argmin} & \norm{M}_{\max}  \\
            \text{s.t. } & \indic_{\rho_t} \circ M = \indic_{\rho_t} \circ Y_t, 
            \quad \norm{M}_\infty \le L_t.
            \end{split}
            \label{eq:nuclear_norm_min_infinite}
        \end{align}
        We impose an entrywise equality constraint because the information on the support of $\rho_t$ is assumed to be noiseless and naturally we want the solution to exactly fit those entries. 
        }
        We have the following performance guarantee. 
        \xxedit{The proof of Theorem~\ref{thm:dt_bound} involves two steps. We first decompose the evaluation error as a summation of the matrix estimation accuracy from future timesteps and then bound the accuracy at each timestep by a standard application of H\"older's inequality.} The complete proof is deferred to Appendix~\ref{appen:proof_dt_bound}. 
        \begin{thm}[Infinite samples]
        \label{thm:dt_bound}
        In the infinite-sample setting, under Algorithm~\ref{alg:MC_low-rank_infinite}
        with $\rho_t = d_t^{\behavior}$ and $\ME(\cdot)$ being~\eqref{eq:nuclear_norm_min_infinite}, the output estimator $\widehat J$ satisfies
        \begin{equation} 
        \label{eq:master_bound_noiseless}
            \big| \widehat{J} - J^{\pi^{\theta}} \big|
            \le 
            2 H \sqrt{dSA} \textstyle \sum_{t=1}^{H} \dis( d_{t}^{\behavior}, d_{t}^{\target}).
        \end{equation}
        \end{thm}
        \xxedit{In the above bound, note that the operator discrepancy only depends on the support of $d_t^{\behavior}$, not the specific distribution. This makes sense since the information of the data entirely depends on the support of $d_t^{\behavior}$, not the specific distribution. 
        Once a state-action pair is supported, we know exactly what $r_t(s,a)$ and $P_t(\cdot \vert s,a)$ are, and therefore it does not matter what the actual value of $d_t^{\behavior}(s,a)$ is. 
        When the support of $d_t^{\behavior}$ is $\statespace \times \actionspace$ for all $t\in[H]$, it means the behavior policy is extremely exploratory and covers the whole state-action space. Consequently, we get a zero estimation bound in~(\ref{eq:master_bound_noiseless}) because we know the MDP exactly. 
        }

    \subsection{Finite-sample setting}
    
        Next consider the setting with a finite dataset $\dataset =\{ (s_{t}^{k},a_{t}^{k},r_{t}^{k})\} _{t\in[H], k\in[K]}$. 
        Let $n_{t}(s,a)\coloneqq\sum_{k\in[K]}\indic_{ (s_{t}^{k},a_{t}^{k})=(s,a) }$ be the visitation count of each state-action pair.
        Accordingly, the empirical occupancy of $\behavior$ is given by  $\widehat{d}_t^{\behavior}(s,a) = n_{t}(s,a)/K$. 
        \xxedit{Let $\rho_t = \widehat{d}_t^{\behavior}$ in Algorithm~\ref{alg:MC_low-rank_infinite}. The matrix estimation subroutine $\ME(\cdot)$ is given by the following max norm minimization program:
        \begin{align}
                \begin{split}
                \ME(\rho_t,  Z_t ) = \underset{M \in\R^{\statesize\times\actionsize}}{\argmin} & \norm{M}_{\max}   \\
                \text{s.t. } & \card{ \ip{\rho_t, M - Z_t} } \le  \card{ \ip{\rho_t, Z_t-Y_t} },
                \quad  \norm{M}_\infty \le L_t .
                \end{split}
                \label{eq:max_norm}
        \end{align}
        We state the following guarantee. \xxedit{The proof of Theorem~\ref{thm:dt_bound_sample} proceeds as follows. We build upon the proof of Theorem~\ref{thm:dt_bound} to get the first discrepancy term on the RHS of~\eqref{eq:noisy_master_bound}. The second error term comes from the finite-sample error in the system and is obtained by applying a generalization error guarantee from~\cite{srebro2005rank}. }
        The complete proof is deferred to Appendix~\ref{appen:proof_dt_bound_sample}. 
         \begin{thm}[Finite samples]
                \label{thm:dt_bound_sample}
                Consider the finite-sample setting under Algorithm~\ref{alg:MC_low-rank_infinite} with $\rho_t = \widehat{d}_t^{\behavior}$ and $\ME(\cdot)$ being~(\ref{eq:max_norm}). We assume $2 < K < SA$. There exists an absolute constant $C>0$ such that with probability at least $1-\delta$, we have
                \begin{equation} 
                \label{eq:noisy_master_bound}
                    \begin{split}
                        \big| \widehat{J} - J^{\pi^{\theta}} \big|  
                        \le &   2 H\sqrt{dSA} \sum_{t\in[H]} \dishat( d_t^{\behavior},  d_{t}^{\pi_{\theta}} ) + C H^2 \sqrt{\frac{d (S+A) \log( HS/\delta) }{K}} . 
                    \end{split}
                \end{equation}
            \end{thm}
        On the RHS of~\eqref{eq:noisy_master_bound}, the first term quantifies the population-level distribution shift and the second term takes into account the statistical error. 
        The finite-sample distribution shift is reflected in the term $\dishat( d_t^{\behavior},  d_{t}^{\pi_{\theta}} )$, which is always finite, given any behavior policy $\behavior$ and target policy $\target$, in contrast to the concentrability coefficient $C^\pi =  \max_{s,a} \frac{d^\pi(s,a)}{\rho(s,a)} $. Suppose that there exists some $(s,a)$ such that $\rho_t(s,a)=0$ and $d_t^{\target} (s,a) > 0$. Then, $C^{\target} = \infty$ whereas $\dishat( d_t^{\behavior}, d_{t}^{\pi_{\theta}} )$ is finite and meaningful. We will subsequently present examples to showcase the effectiveness of our bound. 
        }

        \cycomment{When is this bound nontrivial? when $\sum_{t\in[H]} \dishat( d_t^{\behavior} \leq 1/2 \sqrt{dSA}$ ... is there a more general way to think about what scaling $\dishat( d_t^{\behavior}$ could have? What is the worst case bound on the discrepancy? i.e. even though it's finite, if it's larger than $ 1/2 \sqrt{dSA}$, then we still have to be careful?}
        
        \xxedit{
        As a sanity check, let \cyedit{us consider the setting of evaluating the behavior policy, i.e.} $\target = \behavior$. Using our results, we obtain an error bound of
        \[
            \card{\widehat J - J^{\behavior}} \lesssim H^2 \sqrt{\frac{d(S+A) \log( HS / \delta)}{K}},
        \]
        with probability at least $1-\delta$. 
        This implies that for evaluating the behavior policy, our method requires a sample complexity of order $H^4 d (S+A)$, which matches the standard linear dependence on the dimensions in low-rank matrix estimation. }

        \subsection{Examples}
            We present some concrete examples showcasing the effectiveness of our algorithm. 

            \subsubsection{Policies with Disjoint Support under Uniform Transitions}
            \label{subsubsec:disjoint_support}
                Assume $\statesize=\actionsize= n$. Consider the simple setting where
                the transition is uniform over all state-action pairs.
                For each $s$ and $t$, assume $\pi_{t}^{\theta}(\cdot|s)$
                \cyreplace{is supported on $m$ actions, and the locations of these actions are a realization of uniform random sampling over $[n]$.}{selects an action uniformly at random amongst a subset of actions $\actionspace^{\theta}_t \subseteq \actionspace$, where $|\actionspace^\theta_t| = m$, and the subset $\actionspace^\theta_t$ is itself sampled uniformly at random amongst all subsets of size $m$.}
                We assume $\pi_{t}^{\beta}$ is generated from the same model independently, \cyedit{i.e. the behavior policy also randomizes uniformly amongst a subset of actions $\actionspace^\beta_t$, for a uniformly selected subset of actions}.
                Note that the support of $d_{t}^{\pi^{\beta}}$ and $d_{t}^{\pi^{\theta}}$
                will be mostly disjoint since $\big\vert \mathcal{A}_t^{\theta} \cap \mathcal{A}_t^{\beta} \big\vert$ can be very small, making the concentrability coefficient infinite with high probability. Using Theorem~\ref{thm:dt_bound}, we derive the following infinite-sample error bound, the proof of which is deferred to Appendix~\ref{appen:proof_cor_example}.
                \begin{corollary}
                    \label{cor:example}
                    Under the aforementioned setting, there exists an absolute constant $C>0$ such that when $n\ge C$, with probability at least $1-\frac{1}{n}$, we have 
                    \begin{align}
                    \label{eq:disjoint_support_infinite}
                        \big|\widehat{J}-J^{\pi^{\theta}}\big|  \le C
                    H^2 \sqrt{\frac{d \log (nH) }{m}}.
                    \end{align}
                \end{corollary}
                
                \xxedit{If $m$ satisfies $m \gtrsim  \frac{H^2 d \log(nH) }{\epsilon^2}$ for some $\epsilon>0$, then we have $|\widehat{J}-J^{\pi^{\theta}}| \le \epsilon H$. It implies that even when $m$ is logarithmic in $n$, we can still achieve a consistent error bound. }
                For example, suppose $m=n/2$. In this setting, the behavior and target policies
                both randomize over half of the actions, but their actions may have
                little overlap. Our bound gives $|\widehat{J}-J^{\pi^{\theta}}|\lesssim H^2 \sqrt{\frac{d \log (nH) }{n}},$
                which can be vanishingly small when $n$ is large. \xxedit{The bound~\eqref{eq:disjoint_support_infinite} identifies the inherent difficulty of distribution shift, manifested as a quantity proportional to $H^2 \sqrt{\frac{d}{m}}$, ignoring the logarithmic factor. When $d$ and $H$ are fixed, we have a larger estimation error when $m$ is small, which is expected since small $m$ indicates little to no overlap between $d_t^{\target}$ and $d_t^{\behavior}$.}
    
                \xxedit{For the finite-sample case, we apply Theorem~\ref{thm:dt_bound_sample} and get the following corollary, the proof of which is deferred to Appendix~\ref{appen:proof_disjoint_support_finite}. 
                \begin{corollary}
                \label{cor:example_disjoint_support_finite-sample}
                    Under the same setting as in Corollary~\ref{cor:example}, there exists an absolute constant $C>0$ such that when $n \ge C$, with probability at least $1-\frac1n$, we have 
                    \begin{align}
                        \label{eq:disjoint_support_finite-sample}
                        \card{\widehat J - J^{\target}} \le  C H^2 \left( \sqrt{\frac{d \log(nH)}{m} }  + \sqrt{\frac{d n \log(nH)}{K}}\right) .
                    \end{align}
                \end{corollary}}
                \xxedit{Interestingly, the first term on the RHS of~(\ref{eq:disjoint_support_finite-sample}) will dominate if $K \gtrsim nm$, \cyedit{i.e. when there is at least a constant number of samples per state-action pair in the support of the behavior policy}. 
                This indicates that when $K$ is sufficiently large or $m$ is small enough, the population-level distribution shift will become the main source of estimation error. }

        \subsubsection{Contextual Bandits}
        \label{subsubsec:contextual_bandit}
        
            \xxedit{In this section, we consider specializing our results to the problem of contextual bandits. Specifically, we consider $H=1$.
            In this case, the states are the contexts and the agent acts based on the given contexts. 
            Theorem~\ref{thm:dt_bound} yields the following corollary. For notation simplicity, let $d^{\target} \equiv d_1^{\target}$ and $d^{\behavior} \equiv d_1^{\behavior}$. 
            \begin{corollary}[Infinite samples with $H=1$]
                \label{cor:inf_bound_h=1}
                In the infinite-sample setting, under Algorithm~\ref{alg:MC_low-rank_infinite}
                with $\rho = d^{\behavior}$ and $\ME(\cdot)$ being~\eqref{eq:nuclear_norm_min_infinite}, the output estimator $\widehat J$ satisfies
                \begin{equation} 
                \label{eq:master_bound_noiseless_bandit}
                    \big| \widehat{J} - J^{\pi^{\theta}} \big|
                    \le 
                    2  \sqrt{dSA}  \dis( d^{\behavior}, d^{\target}).
                \end{equation}
            \end{corollary}}

            \xxedit{The following result deals with finite-sample setting, which is a direct corollary of Theorem~\ref{thm:dt_bound_sample}. 
            \begin{corollary}[Finite samples with $H=1$]
                \label{cor:fin_bound_h=1}
                Consider the finite-sample setting under Algorithm~\ref{alg:MC_low-rank_infinite} with $\rho = \widehat{d}^{\behavior}$ and $\ME(\cdot)$ being~\eqref{eq:max_norm}. 
                There exists an absolute constant $C>0$ such that with probability at least $1-\delta$, we have
                \begin{equation} 
                \label{eq:noisy_master_bound_bandit}
                    \begin{split}
                        \big| \widehat{J} - J^{\pi^{\theta}} \big|  
                        \le &  2 \sqrt{dSA} \dishat(d^{\behavior}, d^{\target})  + 
                        C \sqrt{\frac{ d (S+A)  \log( H S/\delta)} {K } }.
                    \end{split}
                \end{equation}
            \end{corollary}}

            \xxedit{We now analyze the operator discrepancy between $d^{\behavior}$ and $d^{\target}$.
            Recall that the environment has a fixed initial state distribution, $\mu \in \Delta(\statespace)$. For all policy $\pi$, the state-action occupancy measure $d^{\pi}$ can be written as $ d^{\pi} (s,a) = \mu(s) \pi (a\vert s).$
            If we view $\mu \in \R^{S}$ as a vector and $d^{\pi}, \pi \in \R^{S \times A}$ as matrices, we can write $d^{\pi} = (\mu \mathbf{1}^\top ) \circ  \pi,$
            where $\mathbf{1} \in \R^{S}$ is an all-one vector.
            Under this notation, we have
            \begin{align*}
                \opnorm{d^{\target} - d^{\behavior}} & = \opnorm{(\mu \mathbf{1}^\top ) \circ (\target - \behavior)} \\
                & \le \norm{\mu}_\infty \opnorm{\target - \behavior},
            \end{align*}
            since $\mu \mathbf{1}^\top$ is a rank-$1$ matrix.
            Hence, inequality~\eqref{eq:noisy_master_bound_bandit} can be written as
            \begin{align*}
                 \big| \widehat{J} - J^{\pi^{\theta}} \big|  
                        \le &  2 \sqrt{dSA} \norm{\mu}_\infty \opnorm{\target - \behavior}  + 
                        C \sqrt{\frac{ d (S+A)  \log( H S/\delta)} {K } }.
            \end{align*}
            For the infinite-sample case, we consider an arbitrary policy $\pi$ that satisfies $\supp(\pi) \subseteq \supp(\behavior)$. Consequently, we have $\supp( d^{\pi}) \subseteq \supp(d^{\behavior})$ and $\dis(d^{\behavior}, d^{\target}) \le \opnorm{d^\pi- d^{\target}}$ as a result. With a slight abuse of notation, we denote the operator discrepancy between two policies as $$\dis(\behavior, \target) = \min \left\{  \opnorm{\pi - \target} : \text{policy }\pi, \supp(\pi) \subseteq \supp(\behavior) \right\}.$$ 
            Thus, the infinite-sample bound~\eqref{eq:master_bound_noiseless_bandit} can be further upper bounded by
            \begin{align*}
                \big| \widehat{J} - J^{\pi^{\theta}} \big|
                    \le 
                    2 \sqrt{dSA} \norm{\mu}_{\infty} \dis( \behavior, \target).
            \end{align*}
            Because of the simplicity of contextual bandits, we are able to transform the discrepancy between distributions to the discrepancy between policies \cyedit{which is easier to directly evaluate}. }


\section{Constrained Off-policy Improvement}
\label{sec:policy_constraint_method}

    In this section, we build on our policy evaluation methods to design an offline policy optimization algorithm. 
    Given a dataset $\dataset$ generated by a behavior policy $\behavior$, we use Algorithm~\ref{alg:MC_low-rank_infinite} to obtain an value estimate $\widehat{J}^{\pi}$ for each policy $\pi$. 
    We then optimize over a subclass of policies for which we can guarantee that the above estimate is reliable. 
    Ideally, we could optimize over the following set of candidate policies 
    $\Pi_{\B}$,
    for which the empirical operator discrepancy, as defined in~(\ref{eq:defn_finite-sample_discrepancy}), between the candidate and behavior policies is bounded above by some parameter 
    $B_t \geq 0$ for all 
    $t \in [H]$,
    \begin{align*}
        \Pi_{\B} \coloneqq \big\{ \pi : \dishat(d_t^{\behavior}, d_t^\pi) \le B_t, \forall t \in [H] \big\}.
    \end{align*}
    \cycomment{curious if we had nonuniform bounds $B_t$?}
    Importantly, when $B_t>0$, the set $\Pi_{\B}$ contains policies with infinite concentrability coefficients, as demonstrated in the example from the last section. 
    With a bigger $B_t$, the set $\Pi_{\B}$ includes more policies, at a price of weaker evaluation guarantees for these policies. 

    \xxedit{Policy constraint/penalty method is prevalent in offline learning to address distribution shift. Researchers have proposed a variety of measures to enforce the constraint. For instance, KL-divergence is a popular choice to make sure the learned policy is close to the behavior policy, as seen in~\cite{peng2019advantage, nair2020awac}. The maximum mean discrepancy~(MMD) also proves to be useful in practice~\cite{kumar2019stabilizing}. However, both KL-divergence and MMD are sensitive to the support shift, whereas our operator discrepancy imposes a milder condition on the difference between the support. }
    
    \xxedit{Determining whether a policy $\pi$ is in $\Pi_{\B}$ is non-trivial \cyedit{as computing the empirical operator discrepancy requires knowledge of the transition dynamics}. In practice, we can instead optimize over a smaller set of candidate policies $\widetilde{\Pi}_{\B} \subseteq \Pi_{\B} $, for which determining $\widetilde{\Pi}_{\B} $ is feasible; \cyedit{we will subsequently illustrate a construction for $\widetilde{\Pi}_{\B} $ via an $\varepsilon$-net of the set specified in \eqref{eq:candidate_policy_set}.} } Among all candidate policies in $\widetilde{\Pi}_{\B} $, we maximize the estimated  values obtained by Algorithm~\ref{alg:MC_low-rank_infinite} to get 
    \begin{equation}
        \pihat = \underset{\pi \in \widetilde{\Pi}_{\B} }{\argmax} \; \widehat{J}^{\pi}.
        \label{eq:policy_opt}
    \end{equation}
    We present the following guarantee for $\pihat$, the proof of which can be found in Appendix~\ref{appen:proof_policy_opt}.

     \begin{thm}
        Suppose $\behavior \in \widetilde{\Pi}_{\B} $. 
        We obtain $\pihat$ by solving~(\ref{eq:policy_opt}). There exists an absolute constant $C>0$ such that with probability at least $1-\delta$, we have
        \begin{equation}
                J^{\pihat} \ge  J^{\pi} - 4 H \sqrt{dSA} \sum_{t \in [H]} B_t  - C \sqrt{\frac{ d (S+A)  \log( \vert \widetilde{\Pi}_{\B} \vert H S/\delta)} {K } }, \quad \forall \pi \in \widetilde{\Pi}_{\B}.
            \label{eq:policy_opt_bound}
        \end{equation}
        \label{thm:policy_opt}
    \end{thm}

    The above bound shows that we are able to find a policy $\pihat$ with a nearly optimal value, compared to other policies in $\widetilde{\Pi}_{\B}$. 
    How close $\pihat$ is to the optimal policy in $\widetilde{\Pi}_{\B}$ depends on how accurately we can evaluate all policies in $\widetilde{\Pi}_{\B}$. 
    According to Theorem~\ref{thm:dt_bound_sample}, the estimations are accurate if $\sum_{t\in[H]} B_t$ is small (policies are close to behavior) and $K$ is large (dataset is large), which is reflected in the bound~(\ref{eq:policy_opt_bound}). 
    Similarly as before, the two error terms in~(\ref{eq:policy_opt_bound}) quantify the fundamental difficulty of distribution shift and finite-sample noise, respectively.
    \cycomment{Should the appropriate bound be one in which the discrepancy bound is relative to the suboptimality gap?}

    \xxedit{One way to find such a subset $\widetilde{\Pi}_{\B}$ is by constraining the policy directly. We do this with the help of the following lemma, which implies two policies that are close in terms of operator norm difference produce state-action occupancy measures close in empirical operator discrepancy. The proof of Lemma~\ref{lem:distritbuion_difference} can be found in Appendix~\ref{appen:proof_distirbution_difference}.}

    \xxedit{
    \begin{lem}
    \label{lem:distritbuion_difference}
        For an arbitrary pair of policies $\target=\{ \target_t \}_{t \in [H]}, \behavior=\{\behavior_t\}_{t \in [H]}$, we have 
        \begin{align}
             \dishat ( d_t^{\target} , d_t^{\behavior}) \le \sum_{i=1}^t \left(\sqrt{dS^2A }\right)^{t-i} \opnorm{\target_i - \behavior_i}, \quad \forall t \in [H]. 
        \end{align}
    \end{lem}
    Following Lemma~\ref{lem:distritbuion_difference}, we can define $\widetilde{\Pi}_{\B}$ as a finite subset (e.g.~an $\varepsilon$ net) of the following set:
    \begin{align} \label{eq:candidate_policy_set}
        \bigg\{ \pi : \opnorm{\target_t - \behavior_t} \le B_t \left( \sqrt{dS^2 A} \right)^{t-H}, \forall t \in [H]  \bigg\}.
    \end{align}
    For all $\pi \in  \widetilde{\Pi}_{\B} $, we have $\dishat ( d_t^{\target} , d_t^{\behavior}) \le B_t$ for all $t\in[H]$, indicating $\pi \in \Pi_{\B}$. 
    The exponential factor {\footnotesize$\left( \sqrt{dS^2 A} \right)^{t-H}$} restricts the candidate policies to be exceedingly close to the behavior policy, especially at earlier stages. Intuitively, this makes sense since if the policies at early stages are too different, the resulting deviation will be amplified in later steps of the horizon, resulting in the exponentially \cyedit{growing multiplicative factor in the discrepancy bound}. 
    }


\section{Conclusion}
\label{sec:discussion}

    We propose a novel algorithm for efficient offline evaluation when low-rank structure is present in the MDP. 
    Our algorithm is a combination of Q iteration and low-rank matrix estimation, which is easy to implement. 
    We show that the proposed operator discrepancy measure better captures the difficulty of policy evaluation in the offline setting, compared to the traditional concentrability coefficient. We also combine the evaluation algorithm with policy optimization and provide performance guarantee. 
    We believe that this work is a first step in exploiting the benefit of low-rank structure in the Q function in offline RL.
    Future directions of interest include extending our results to the infinite-horizon setting with stationary policies, and understanding lower bounds for estimation that would provide insight on whether or not our estimation error bounds are optimal. 

\paragraph{Acknowledgement:}  C.\ Yu is partially supported by NSF grants CCF-1948256 and CNS-1955997, AFOSR grant FA9550-23-1-0301, and by an Intel Rising Stars award. Y.\ Chen is partially supported by NSF grants CCF-1704828 and CCF-2047910.

\bibliographystyle{plain} 
\bibliography{ref}

\newpage

\appendix

\section{Proofs} 
\label{appen:proof}

    Let $\mu_{t}^{\pi}:\statespace\to[0,1]$
    denote the state occupancy measure at time $t\in[H]$ under policy $\pi$.
    

    \subsection{Proof of Theorem~\ref{thm:dt_bound}} 
    \label{appen:proof_dt_bound}
        We present two lemmas before analyzing the evaluation error. 
        The first one analyzes the error incurred at the matrix estimation step. The proof is deferred to Appendix~\ref{appen:proof_matrix_diff_bound}. 
        \begin{lem}
            \label{lem:matrix_diff_bound}For arbitrary real matrices $A,B,P,W\in\R^{m \times n}$,
            we have 
            \begin{equation*}
                \card{\sum_{i,j}W_{ij}(A_{ij}-B_{ij})}  \le \card{\sum_{i,j}P_{ij}(A_{ij}-B_{ij})}+\left(\nunorm A+\nunorm B\right)\opnorm{P-W}.
            \end{equation*}
        \end{lem}
        \begin{rem*}
            Under the matrix estimation framework, we can interpret matrix $P$ as the sampling pattern and $W$ as the weights for evaluation. 
        \end{rem*}
    
        Next, we introduce a lemma decomposing the evaluation error as a summation of the matrix estimation accuracy from future timesteps. The proof can be found in Appendix~\ref{appen:proof_error_decomp}.
        \begin{lem}
        \label{lem:error_decomp}
        For the Q function and its estimator  $Q_{t}^{\target},\widehat{Q}_{t}^{\target} \in\R^{\statesize\times\actionsize}$,
        we have 
        \[
        \left\langle d_{t}^{\pi^{\theta}},\widehat{Q}_{t}^{\target}-Q_{t}^{\pi^{\theta}}\right\rangle =\left\langle d_{t}^{\pi^{\theta}},\widehat{Q}_{t}^{\target}-Y_{t}\right\rangle +\left\langle d_{t+1}^{\pi^{\theta}},\widehat{Q}_{t+1}^{\target}-Q_{t+1}^{\pi^{\theta}}\right\rangle ,\;t \in [H],
        \]
         and consequently
        \[
        \left\langle d_{1}^{\pi^{\theta}},\widehat{Q}_{1}^{\target}-Q_{1}^{\pi^{\theta}}\right\rangle =\sum_{t=1}^{H}\left\langle d_{t}^{\pi^{\theta}},\widehat{Q}_{t}^{\target}-Y_{t}\right\rangle .
        \]
        \end{lem}
        
        Based on Lemma~\ref{lem:matrix_diff_bound} and~\ref{lem:error_decomp},
        we derive the following error bound.
        For each $t\in [H]$ and an arbitrary $g_t  \in \Delta(\statesize \times \actionsize)$ with $\supp(g_t) \subseteq \supp( d_t^{\behavior} )$, we have
        \begin{align*}
         & \card{\sum_{s,a}d_{t}^{\target}(s,a)\left(\widehat{Q}_{t}^{\target}(s,a)-Y_{t} (s,a)\right)}\\
        \le & \card{\sum_{s,a} g_t(s,a) \left(\widehat{Q}_{t}^{\target}(s,a)-Y_{t} (s,a)\right)} + \left( \nunorm{Y_{t}} + \nunorm{\widehat{Q}_t^{\target}} \right) \opnorm{d_{t}^{\target}-g_t} \\
        = &  \left( \nunorm{Y_{t}} + \nunorm{\widehat{Q}_t^{\target}} \right) \opnorm{d_{t}^{\target}-g_t} \cycomment{\text{wait is this really an equality?}} \\
        \le &  \sqrt{SA} \left( \norm{Y_t}_{\max} +\norm{\widehat{Q}_t^{\target}}_{\max} \right) \opnorm{d_{t}^{\target}-g_t} \\
        \le  & 2\sqrt{SA} \norm{Y_t}_{\max} \opnorm{d_{t}^{\target}-g_t}.   \\
        \le & 2 H \sqrt{dSA} \opnorm{d_{t}^{\target}-g_t}. 
        \end{align*}
        where the first step follows from Lemma~\ref{lem:matrix_diff_bound},
        the second equality follows from the constraints in~(\ref{eq:nuclear_norm_min_infinite}) and the remaining steps follow from properties of the max norm~\eqref{eq:maxnorm_infty_bound} and  \eqref{eq:maxnorm_nuclearnorm}. \cycomment{Fix this reference.} Then, combining with the decomposition in
        Lemma~\ref{lem:error_decomp}, we obtain the desired bound by minimizing over all such $g_t$. 

    \subsection{Proof of Theorem~\ref{thm:dt_bound_sample}} \label{appen:proof_dt_bound_sample}
        Fix $t\in[H]$. The solution $\widehat Q_{t}^{\target}$ satisfies 
            \begin{align}
                \norm{ \widehat Q_{t}^{\target}}_{\max} &\le \norm{Y_t}_{\max} \le \sqrt d H \label{eq:Qhat_max_norm_bound} \\
                \card{ \ip{\rho_t, \widehat Q_{t}^{\target} - Z_t} } &\le  \card{ \ip{\rho_t, Z_t-Y_t} } \label{eq:Qhat_inner_product_bound}
            \end{align}
            We apply Theorem~6 in~\cite{srebro2005rank} with loss function $g(x;y) = x-y$, target matrix $Y = Y_t$, distribution $\mathcal{P} = d_t^{\behavior}$. The discrepancy weighted by $\mathcal{P}$ corresponds to $\ip{ d_t^{\behavior}, \widehat Q_t^{\target}  - Y_t}$ and the empirical discrepancy is $\ip{ \rho_t, \widehat Q_t^{\target}- Y_t}$.  With probability at least $1-\frac{\delta}{2H}$,  we obtain
            \begin{align}
            \label{eq:me_generalization_bound}
                \ip{ d_t^{\behavior}, \widehat Q_t^{\target}  - Y_t} & \le \ip{ \rho_t, \widehat Q_t^{\target}- Y_t} + 17 \sqrt{\frac{dH^2 (S+A) + \log (2H / \delta ) }{K}},
            \end{align}
            where we use~\eqref{eq:Qhat_max_norm_bound}. 
            Hence, the estimation error can be upper bounded by
            {\footnotesize
            \begin{align*}
                \card{\widehat J - J^{\target}} &\le \sum_{t\in[H]} \card{\ip{ d_t^{\target}, \widehat Q_t^{\target}  - Y_t} } \\
                & \overset{\text{(i)}}{\le} \sum_{t\in[H]} \card{ \ip{ d_t^{\behavior}, \widehat Q_t^{\target}  - Y_t} } + 2  \sum_{t\in[H]} (\nunorm{Y_{t}} +\nunorm{\widehat{Q}_t^{\target}}) \cdot\opnorm{d_{t}^{\pi_{\theta}} - d_t^{\behavior}}  \\
                &\overset{\text{(ii)}}{\le} \sum_{t\in[H]} \card{\ip{\rho_t, \widehat Q_t^{\target} - Y_t}} + 17 H \sqrt{\frac{dH^2 (S+A) + \log (2H / \delta ) }{K}} + 2H \sqrt{dSA}  \sum_{t\in[H]} \opnorm{d_{t}^{\pi_{\theta}} - d_t^{\behavior}} \\
                &\overset{\text{(iii)}}{\le} 2\sum_{t\in[H]} \card{\ip{\rho_t, \widehat Z_t - Y_t}} + 17 H \sqrt{\frac{dH^2 (S+A) + \log (2H / \delta ) }{K}} +  2 H \sqrt{dSA} \sum_{t\in[H]} \opnorm{d_{t}^{\pi_{\theta}} - d_t^{\behavior}},
            \end{align*}
            }
            where step~(i) invokes Lemma~\ref{lem:matrix_diff_bound}; step~(ii) follows from~\eqref{eq:me_generalization_bound}, and properties of the max norm~\eqref{eq:maxnorm_infty_bound} and  \eqref{eq:maxnorm_nuclearnorm}; step~(iii) uses the triangle inequality and ~\eqref{eq:Qhat_inner_product_bound}. 
            Applying Lemma~\ref{lem:empirical_error} and the union bound, we conclude that
            \begin{align*}
                \card{\widehat J - J^{\target}} &\le 2 C H^2 \sqrt{\frac{S \log ( HS / \delta) }{K}} + 17 H \sqrt{\frac{ dH^2 (S+A) + \log (2H / \delta ) }{K}} + 2 H\sqrt{dSA} \sum_{t\in[H]} \opnorm{d_{t}^{\pi_{\theta}} - d_t^{\behavior}} \\
                & \le 2 H\sqrt{dSA} \sum_{t\in[H]} \opnorm{d_{t}^{\pi_{\theta}} - d_t^{\behavior}}  + C' H^2 \sqrt{\frac{d (S+A) \log(HS/\delta)}{K}} \\
                & \le 2 H\sqrt{dSA} \sum_{t\in[H]} \opnorm{d_{t}^{\pi_{\theta}} - d_t^{\behavior}} + C' H^2 \sqrt{\frac{d (S+A) \log(HS/\delta)}{K}},
            \end{align*}
            with probability at least $1-\delta$.

        \begin{lem}
        \label{lem:empirical_error}
            There exists an absolute constant $C>0$ such that with probability at least $1-\delta$, we have
            \begin{align}
                 \card{\ip{\rho_t, Z_t - Y_t} } \le C H \sqrt{\frac{S \log ( HS / \delta) }{K}}, \quad \forall t \in [H]. 
            \end{align}
        \end{lem}
        \begin{proof}
            Fix $t\in[H]$. Recall that
            \[
                Z_t(s,a) - Y_t(s,a) =  \sum_{s',a'} (\widehat{P}_t(s' \vert s,a) - P_t(s'\vert s,a)) \target_{t+1}(a' \vert s') \widehat{Q}_{t+1}^{\target}(s',a').
            \]
            To further analyze the above expression, we write down an explicit formula for the empirical transition probability:  $\widehat P_t(s' \vert s,a)$ is obtained by
            \[
                \widehat P_t(s' \vert s,a) = \frac{1}{n_t(s,a)} \sum_{k=1}^K \indic_{\{(s_t^k, a_t^k, s_{t+1}^k) = (s,a,s')\}}.
            \]
            Under the above notations, we get
            {\footnotesize
            \begin{align*}
                 Z_t(s,a) - Y_t(s,a) 
                 = & \sum_{s'} \left( \frac{1}{n_t(s,a)}  \sum_{k=1}^K \indic_{\{(s_t^k, a_t^k, s_{t+1}^k) = (s,a,s')\}} -  P_t(s' \vert s,a) \right)  \sum_{a'}  \target_{t+1}(a' \vert s') \widehat{Q}_{t+1}^{\target}(s',a').
            \end{align*}
            }
            For simplicity, define $f_t(s') = \sum_{a'}  \target_{t+1}(a' \vert s') \widehat{Q}_{t+1}^{\target}(s',a')$. It is guaranteed that $\card{ f_t(s')} \le H$.
            Invoke the identity $\rho_t(s,a) = \frac{n_t(s,a)}{K}$  and  we obtain
            \begin{align*}
                 & \sum_{s,a} \rho_t(s,a) \left( Z_t(s,a)-Y_{t}(s,a) \right) \\
                 = &\underbrace{\frac{1}{K}\sum_{s'} f_t (s')
                 \sum_{k=1}^{K}\sum_{s,a}\left( \indic_{\left\{ (s_{t}^{k},a_{t}^{k})=(s,a)\right\} } \indic_{\left\{ s_{t+1}^{k}=s'\right\} } - \indic_{\left\{ (s_{t}^{k},a_{t}^{k})=(s,a)\right\} }P_{t}(s'\vert s,a)\right)}_{\alpha}.
            \end{align*}
            For $\alpha$, we first use the fact that $f_t(s')$ is bounded to derive that
            \begin{align*}
                \card{\alpha} \le  \frac{H}{K}\sum_{s'}\left|\sum_{k=1}^{K}\sum_{s,a}\left( \indic_{\left\{ (s_{t}^{k},a_{t}^{k})=(s,a)\right\} } \indic_{\left\{ s_{t+1}^{k}=s'\right\} } - \indic_{\left\{ (s_{t}^{k},a_{t}^{k})=(s,a)\right\} }P_{t}(s'\vert s,a)\right)\right|. 
            \end{align*}
            Let $X_{k}^{s'}=\sum_{s,a}\left( \indic_{\left\{ (s_{t}^{k},a_{t}^{k})=(s,a)\right\} } \indic_{\left\{ s_{t+1}^{k}=s'\right\} } - \indic_{\left\{ (s_{t}^{k},a_{t}^{k})=(s,a)\right\} }P_{t}(s'\vert s,a)\right)$
            for all $s'$ and $k$. If we fix $s'$, it is easy to see that $ \{ X_{k}^{s'},k\in[K] \}$
            are independent. 
            Even if $X_{k}^{s'}$ is defined as a sum,
            exactly one of the indicators can be non-zero. Hence, we get $\left|X_{k}^{s'}\right|\le2$. 
            However, this upper bound is not enough for a tight concentration.
            To resolve this, we calculate the variance of $X_{k}^{s'}$. Fix $s'$
            and $k=1$. We rewrite $X_{1}^{s'}$ as
            \[
                X_{1}^{s'}=\sum_{s,a} \indic_{\left\{ (s_{t}^{1},a_{t}^{1})=(s,a)\right\} }\left( \indic_{\left\{ s_{t+1}^{k}=s'\right\} }-P_{t}(s'\vert s,a)\right).
            \]
            Define $F_{s,a}= \indic_{\left\{ (s_{t}^{1},a_{t}^{1})=(s,a)\right\} }$,
            which follows a Bernoulli distribution with success probability $d_{t}^{\pi^{\beta}}(s,a)$.
            Define $G_{s'}= \indic_{\left\{ s_{t+1}^{k}=s'\right\} }$, which follows
            a Bernoulli distribution with success probability $\mu_{t+1}^{\pi^{\beta}}(s')$.
            We use the shorthand $P_{s'\vert s,a}=P_{t}(s'\vert s,a).$ Under
            these notations, we get $X_{1}^{s'}=\sum_{s,a}F_{s,a}\left(G_{s'}-P_{s'\vert s,a}\right)$.
            Next, we calculate the variance of $X_{1}^{s'}$ as
            \begin{align*}
                \mathbb{E}\left[\left(X_{1}^{s'}\right)^{2}\right] & =\mathbb{E}\left[\left(\sum_{s,a}F_{s,a}\left(G_{s'}-P_{s'\vert s,a}\right)\right)^{2}\right]\\
                 & \overset{\text{(i)}}{=}\sum_{s,a}\mathbb{E}\left[F_{s,a}^{2}\left(G_{s'}-P_{s'\vert s,a}\right)^{2}\right]\\
                 & =\sum_{s,a}d_{t}^{\pi^{\beta}}(s,a)\left[\mu_{t+1}^{\pi^{\beta}}(s')\left(1-P_{s'\vert s,a}\right)^{2}+\left(1-\mu_{t+1}^{\pi^{\beta}}(s')\right)P_{s'\vert s,a}^{2}\right]\\
                 & =\sum_{s,a}d_{t}^{\pi^{\beta}}(s,a)\left[\mu_{t+1}^{\pi^{\beta}}(s')-2\mu_{t+1}^{\pi^{\beta}}(s')P_{s'\vert s,a}+P_{s'\vert s,a}^{2}\right]\\
                 & \overset{\text{(ii)}}{\le}\mu_{t+1}^{\pi^{\beta}}(s')-2\left[\mu_{t+1}^{\pi^{\beta}}(s')\right]^{2}+\mu_{t+1}^{\pi^{\beta}}(s')\\
                 & =2\mu_{t+1}^{\pi^{\beta}}(s')\left[1-\mu_{t+1}^{\pi^{\beta}}(s')\right]\le2\mu_{t+1}^{\pi^{\beta}}(s'),
            \end{align*}
            where step~(i) ignores all cross terms since $F_{s,a}F_{\bar{s},\bar{a}}=0$
            for $(s,a)\neq(\bar{s},\bar{a})$; step~(ii) follows from $P_{s'\vert s,a}^{2}\le P_{s'\vert s,a}$.
            Therefore, we can control the sampling error by applying Bernstein's inequality over the sums $\frac1K \sum_{k=1}^K X_k^{s'}$ for all $s'$: 
            \begin{align*}
                \left| \alpha \right| & \le C H \sqrt{\frac{\log( HS/\delta)}{K}}\sum_{s'}\sqrt{\mu_{t+1}^{\pi^{\beta}}(s')}\le C H \sqrt{\frac{S\log( HS/\delta)}{K}},
            \end{align*}
             with probability at least $1-\frac{\delta}{H}$. The last step follows directly from AM-QM inequality. Therefore, we obtain
             \[
                \card{\ip{\rho_t, Z_t - Y_t} } \le C H \sqrt{\frac{S \log ( HS / \delta) }{K}},
             \]
            with probability at least $1-\frac{\delta}{H}$. 
            Applying the union bound over all $t\in[H]$ yields the desired result. 
        \end{proof}

    \subsection{Proof of Corollary~\ref{cor:example}} 
    \label{appen:proof_cor_example}
        For simplicity, define $b \coloneqq \frac{m}{n}$.
        Since the transition is uniform, the state occupancy $\mu_{t}^\pi (\cdot) $ is uniform under
        any policy $\pi$, i.e.~$\mu_t^{\pi}(s) = \frac{1}{n}$.
        By the way the policies are generated, $d_{t}^{\pi^{\theta}}(\cdot,\cdot)=\mu_{t}^{\pi^{\theta}}(\cdot)\pi_{t}^{\theta}(\cdot|\cdot)\in\R^{n \times n}$
        is supported on $mn$ entries whose locations are realization of random
        sampling, and on these entries $d_{t}^{\pi^{\theta}}(s,a)=\frac{1}{mn}.$ Specifically, all $d_t^{\target}(s,a)$ are i.i.d.~Bernoulli random variables that take the value $1$ with probability $b$. The behavior policy $\behavior$ is generated independently via the same process.
        Let $M \coloneqq d_{t}^{\pi^{\theta}}-d_{t}^{\pi^{\beta}}$ and we have
        \begin{equation}
            M_{ij} = 
            \begin{cases}
                \frac{1}{mn} & \text{with probability } b(1-b) \\
                -\frac{1}{mn} & \text{with probability }  b(1-b) \\
                0 & \text{with probability } 1-2b(1-b)
            \end{cases}
        \end{equation}
        independently across all entries $(i,j)$. 
        By matrix Bernstein inequality, we obtain the following result, the proof of which is deferred to Appendix~\ref{appen:proof_matrix_op_norm_bound}.
        \begin{lem} \label{lem:matrix_op_norm_bound}
            There exists an absolute constant $C > 0$ such that when $n\ge C$, with probability at least $1 - \frac{1}{n}$, we have
            \begin{equation} \label{eq:M_opnorm_bound}
                \opnorm{M} \le C  \sqrt{\frac{\log n }{n^2 m}}.
            \end{equation}
        \end{lem}
        Combining Theorem~\ref{thm:dt_bound}, (\ref{eq:M_opnorm_bound}) and the union bound over all $t\in[H]$, with probability at least $1-\frac1n$, we get
        \begin{align*}
            \left|\widehat{J}-J^{\pi^{\theta}}\right| &\le 2 n H \sqrt{d} \sum_{t=1}^H \opnorm{d_t^{\behavior} - d_t^{\target}} \\
            & \lesssim H \cdot nH \sqrt{d} \cdot \sqrt{\frac{\log(nH)}{n^2  m}} \\
            & = H^2 \sqrt{\frac{d  \log(nH) }{m }},
        \end{align*}
        where the first upper bound is obtained by plugging  $d_t^{\behavior}$ into the objective of~(\ref{eq:defn_dis_pq}).

    \subsection{Proof of Lemma~\ref{lem:matrix_op_norm_bound}} \label{appen:proof_matrix_op_norm_bound}
        We apply matrix Berstein's inequality (Theorem~6.1.1 in~\cite{tropp2015matrix}). 
        Let $S_k = M_{ij} e_i e_j^\top$, for all $k \in [n^2]$. 
        Since $\card{M_{ij}} \le \frac{1}{mn}$, we derive that $\opnorm{S_k} \le \frac{1}{mn}$. 
        We calculate that
        \begin{align*}
            \sum_k \mathbb{E} [S_k S_k^\top] &= \sum_{i,j} \mathbb{E} [M_{ij}^2] e_i e_i^\top \\
            &= 2b(1-b) \frac{1}{m^2 n} I_n.
        \end{align*}
        As a result, we have
        \[
            \opnorm{\sum_k \mathbb{E} [S_k S_k^\top]} = 2b(1-b)\frac{1}{m^2 n}  = \frac{2(n-m)}{n^3 m} \le \frac{2}{n^2m}.
        \]
        By symmetry, we also have $\opnorm{\sum_k \mathbb{E} [S_k^\top S_k]} \le \frac{2}{n^2 m} $. 
        Hence, we get
        \begin{align*}
            \mathbb{P} \left( \opnorm{M}  \ge t \right) & \le 2n \exp \left( \frac{-t^2/2}{\frac{2}{m n^2 } + \frac{t}{3mn} }\right).
        \end{align*}
        Letting the RHS be upper bounded by $\frac{1}{n}$ yields the desired result.

    \subsection{Proof of Corollary~\ref{cor:example_disjoint_support_finite-sample}}
    \label{appen:proof_disjoint_support_finite}
        Since generating $d_t^{\behavior}$ is done independently from the offline data collection, we first condition on $d_t^{\behavior}$ and $d_t^{\target}$. 
d        By Theorem~\ref{thm:dt_bound_sample}, with probability at least $1-\frac{1}{2n}$, we have
        \begin{align*}
            \big| \widehat{J} - J^{\pi^{\theta}} \big|  
            \le &   2 nH \sqrt{d} \sum_{t\in[H]} \opnorm{d_{t}^{\pi_{\theta}} - d_t^{\behavior}} + C H^2 \sqrt{\frac{d n\log( n H)  }{K}}.
        \end{align*}
        Invoking Lemma~\ref{lem:matrix_op_norm_bound}, with probability at least $1-\frac{1}{2n}$, we have $\opnorm{d_{t}^{\pi_{\theta}} - d_t^{\behavior}} \lesssim \sqrt{\frac{\log (nH)}{n ^2 m }}$ for all $t\in[H]$. Applying the union bound, we get
        \begin{align*}
             \big| \widehat{J} - J^{\pi^{\theta}} \big|  
            \lesssim &   H^2 \sqrt{\frac{d \log(nH)}{m} }  + H^2 \sqrt{\frac{d n \log( nH )  }{K}},
        \end{align*}
        with probability at least $1-\frac1n$. 

    \subsection{Proof of Lemma~\ref{lem:matrix_diff_bound}} \label{appen:proof_matrix_diff_bound}
        The proof uses the following result, which holds for any pairs of dual norms. In this paper, we only consider using $\nunorm{\cdot}$ and $\opnorm{\cdot}$.
        \begin{lem}
        \label{lem:p-q}For a real matrix $M\in\R^{m \times n}$ and two weight matrices
        $P,W\in\R^{m \times n}$, we have that 
        \begin{equation*}
             \card{\sum_{i,j}P_{ij}M_{ij}-\sum_{i,j}W_{ij}M_{ij}} \le\nunorm M\opnorm{P-W}.
        \end{equation*}
        \end{lem}
        \begin{proof}
        We can rewrite $\sum_{i,j}P_{ij}M_{ij}-\sum_{i,j}W_{ij}M_{ij}$ as
        \[
        \ip{M,P-W},
        \]
        where $\ip{\cdot,\cdot}$ denotes the trace inner product between
        matrices. Applying H\"older's inequality, we obtain
        \begin{equation*}
            \card{\ip{M,P-W}}  \le\nunorm M\opnorm{P-W}.
        \end{equation*}
        \end{proof}
        Substituing $M_{ij}=A_{ij}-B_{ij}$ in Lemma~\ref{lem:p-q}, we
        immediately obtain the desired results in Lemma~\ref{lem:matrix_diff_bound}.
        
        \subsection{Proof of Lemma~\ref{lem:error_decomp}} \label{appen:proof_error_decomp}
        Recall that 
        \begin{align}
        Q_{t}^{\pi^{\theta}}(s,a) & =\left(B_{t}^{\target}Q_{t+1}^{\pi^{\theta}}\right)(s,a),\label{eq:Bellman}\\
        Y_{t}(s,a) & =\left(B_{t}^{\target}\widehat{Q}_{t+1}^{\pi^{\theta}}\right)(s,a).\label{eq:Y_t}
        \end{align}
        
        For each $(s,a)\in\statespace\times\actionspace$, we have 
        \begin{align*}
        & \widehat{Q}_{t}^{\target}(s,a)-Q_{t}^{\pi^{\theta}}(s,a) \\ =&\left(\widehat{Q}_{t}^{\target}(s,a)-Y_{t}(s,a)\right)+\left(Y_{t}(s,a)-Q_{t}^{\pi^{\theta}}(s,a)\right)\\
         =& \left(\widehat{Q}_{t}^{\target}(s,a)-Y_{t}(s,a)\right)+\sum_{s',a'}P_{t}(s'|s,a)\pi_{t+1}^{\theta}(a'|s')\left(\widehat{Q}_{t+1}^{\target}(s',a')-Q_{t+1}^{\pi^{\theta}}(s',a')\right),
        \end{align*}
        where the last step follows from equations (\ref{eq:Y_t}) and (\ref{eq:Bellman}).
        Multiplying both sides by $d_{t}^{\pi^{\theta}}(s,a)$ and summing
        over $(s,a)$, we obtain 
        \begin{align*}
        & \left\langle d_{t}^{\pi^{\theta}},\widehat{Q}_{t}^{\target}-Q_{t}^{\pi^{\theta}}\right\rangle  \\
        = & \left\langle d_{t}^{\pi^{\theta}},\widehat{Q}_{t}^{\target}-Y_{t}\right\rangle +\sum_{s',a'}\underbrace{\sum_{s,a}d_{t}^{\pi^{\theta}}(s,a)P_{t}(s'|s,a)\pi_{t+1}^{\theta}(a'|s')}_{=d_{t+1}^{\pi^{\theta}}(s',a')}\left(\widehat{Q}_{t+1}^{\target}(s',a')-Q_{t+1}^{\pi^{\theta}}(s',a')\right)\\
         = &\left\langle d_{t}^{\pi^{\theta}},\widehat{Q}_{t}^{\target}-Y_{t}\right\rangle +\left\langle d_{t+1}^{\pi^{\theta}},\widehat{Q}_{t+1}^{\target}-Q_{t+1}^{\pi^{\theta}}\right\rangle ,
        \end{align*}
        thereby proving the first equation in the lemma. Continuing the above
        recursion yields the second equation.

    \subsection{Proof of Lemma~\ref{lem:distritbuion_difference}}
    \label{appen:proof_distirbution_difference}

        The state marginal distribution at time $t$ for an arbitrary policy $\pi$ satisfies $\mu_t^{\pi}(s) = \sum_{s',a'} d_{t-1}^{\pi}(s',a') P_{t-1}(s \vert s',a')$. Accordingly, the state-action occupancy measure satisfies $d_t^{\pi}(s,a) = \sum_{s} \mu_t^{\pi}(s)  \pi_t(a \vert s)$. Equivalently, we write $d^{\pi}_t = ( \mu_t^{\pi} \mathbf{1}^\top ) \circ \pi $, when we view $d^{\pi}_t$ as a $S$-by-$A$ matrix. 
    
        Fix $t\ge 2$, we have
        \begin{align}
            \opnorm{d_t^{\target} - d_t^{\behavior}} & = \opnorm{ (\mu_{t}^{\target} \mathbf{1}^\top) \circ \target_t - (\mu_{t}^{\behavior} \mathbf{1}^\top) \circ \behavior_t } \nonumber \\
            & \le \opnorm{ (\mu_{t}^{\target} \mathbf{1}^\top) \circ \target_t   - (\mu_{t}^{\behavior} \mathbf{1}^\top) \circ \target_t  } + \opnorm{(\mu_{t}^{\behavior} \mathbf{1}^\top) \circ \target_t - (\mu_{t}^{\behavior} \mathbf{1}^\top) \circ \behavior_t} \nonumber \\
            & \le  \sqrt{dS^2 A}
            \opnorm{d_{t-1}^{\behavior} - d_{t-1}^{\target}} + \opnorm{\target_t - \behavior_t}, \label{eq:bound_dt_dt-1}
        \end{align}
        where the first inequality uses the triangle inequality and the second one uses Lemma~\ref{lem:distribution_difference_1} and~\ref{lem:distribution_difference_2}. We take the last display inequality~(\ref{eq:bound_dt_dt-1}) as it is and obtain:
        \begin{align*}
            \opnorm{d_t^{\target} - d_t^{\behavior}} \le \sum_{i=1}^t \left(\sqrt{dS^2A }\right)^{t-i} \opnorm{\target_i - \behavior_i},
        \end{align*}
        for all $t \in [H].$
        Finally, we prove the two lemmas used before. 
        \begin{lem}
        \label{lem:distribution_difference_1}
            For all $t \ge 2$, we have
            \[
                \opnorm{ (\mu_{t}^{\target} \mathbf{1}^\top) \circ \target_t   - (\mu_{t}^{\behavior} \mathbf{1}^\top) \circ \target_t  } \le \sqrt{dS^2 A} \opnorm{d_{t-1}^{\behavior} - d_{t-1}^{\target}}.
            \]
        \end{lem}
        \begin{proof}
            The matrix $\mu_{t}^{\target} \mathbf{1}^\top - \mu_{t}^{\behavior} \mathbf{1}^\top$ is at most rank $2$. Hence, we have
            \begin{align*}
                \opnorm{ (\mu_{t}^{\target} \mathbf{1}^\top) \circ \target_t   - (\mu_{t}^{\behavior} \mathbf{1}^\top) \circ \target_t  } & = \opnorm{ (\mu_{t}^{\target} \mathbf{1}^\top - \mu_{t}^{\behavior} \mathbf{1}^\top) \circ \target_t } \\
                & \le \sqrt 2 \norm{\mu_{t}^{\target} \mathbf{1}^\top - \mu_{t}^{\behavior} \mathbf{1}^\top}_\infty \opnorm{\target_t} \\
                & \le  \sqrt {2 S}  \max_{s} \card{ \mu_{t}^{\target}(s) - \mu_{t}^{\behavior}(s) },
            \end{align*}
            where in the last line we use the fact that $\opnorm{\target_t} \le \sqrt{S}$ since $\target_t$ is a right stochastic matrix (i.e.,\ the sum of each row is $1$). 
            For each state $s$, we have
            \begin{align*}
                \mu_t^{\target}(s) - \mu_t^{\behavior}(s) & =   \sum_{s',a'} \left( d_{t-1}^{\target}(s',a') - d_{t-1}^{\behavior} (s',a') \right) P_{t-1} (s \vert s',a') \\
                & \le \opnorm{d_{t-1}^{\target} - d_{t-1}^{\behavior}} \nunorm{P_{t-1}(s \vert \cdot, \cdot)},
            \end{align*}
            where the last step invokes H\"older's inequality and the notation $P_{t-1}(s \vert \cdot, \cdot)$ denotes a $\statesize \times \actionsize$ matrix by fixing the next state in the transition probability. Note that by assumption, $P_{t-1}(s \vert \cdot, \cdot)$ is at most rank $d/2$. Hence, we further deduce that
            \[
                \nunorm{P_{t-1}(s \vert \cdot, \cdot)} \le \sqrt{\frac{d \statesize \actionsize }{2}}. 
            \]
            Combining pieces, we obtain the desired result. 
        \end{proof}
        \begin{lem}
        \label{lem:distribution_difference_2}
            For all $t$, we have
            \[
                \opnorm{(\mu_{t}^{\behavior} \mathbf{1}^\top) \circ \target_t - (\mu_{t}^{\behavior} \mathbf{1}^\top) \circ \behavior_t} \le  \opnorm{\target_t - \behavior_t}. 
            \]
        \end{lem}
        \begin{proof}
            We use the fact that the matrix $\mu_{t}^{\behavior} \mathbf{1}^\top$ is rank $1$ to deduce that
            \begin{align*}
                \opnorm{(\mu_{t}^{\behavior} \mathbf{1}^\top) \circ \target_t - (\mu_{t}^{\behavior} \mathbf{1}^\top) \circ \behavior_t}& = \opnorm{( \mu_{t}^{\behavior} \mathbf{1}^\top ) \circ (\target_t - \behavior_t)} \\
                & \le \norm{\mu_t^{\behavior}}_{\infty} \opnorm{\target_t - \behavior_t} \le \opnorm{\target_t - \behavior_t}.
            \end{align*}
        \end{proof}

    \subsection{Proof of Theorem~\ref{thm:policy_opt}}
    \label{appen:proof_policy_opt}
        We first apply Theorem~\ref{thm:dt_bound_sample}. For $\pihat$ and an arbitrary policy $\pi \in \widetilde{\Pi}_{\B}$, with probability at least $1-\delta$, we have
        \begin{align*}
            \big| \widehat{J} - J^{\pihat} \big|  
                        & \le 
                        2 H \sqrt{dSA} \sum_{t\in[H]} B_t    + C H^2 \sqrt{\frac{d (S+A) \log( \vert \widetilde{\Pi}_{\B} \vert HS/\delta) }{K}} \\
            \big| \widehat{J} - J^{\pi} \big|  
                        & \le 
                        2  H \sqrt{dSA} \sum_{t\in[H]} B_t   + C H^2 \sqrt{\frac{d (S+A) \log( \vert \widetilde{\Pi}_{\B} \vert  
        HS/\delta) }{K}}. \\
        \end{align*}
        Then, we deduce that
        \begin{align*}
            J^{\pihat} &\ge \widehat{J}^{\pihat}  - 2 H \sqrt{dSA} \sum_{t\in[H]} B_t   - C H^2 \sqrt{\frac{d (S+A) \log( \vert \widetilde{\Pi}_{\B} \vert HS/\delta) }{K}}  \\
            &\ge  \widehat{J}^{\pi}- 2  H   \sqrt{dSA}  \sum_{t\in[H]} B_t  - C H^2 \sqrt{\frac{d (S+A) \log(\vert \widetilde{\Pi}_{\B} \vert HS/\delta) }{K}} \\
            &\ge J^{\pi} - 4  H  \sqrt{dSA}  \sum_{t\in[H]} B_t     - 2 C H^2 \sqrt{\frac{d (S+A) \log( \vert \widetilde{\Pi}_{\B} \vert HS/\delta) }{K}}.
        \end{align*}

\end{document}